\documentclass [11pt]{article} 

\usepackage{fullpage,wrapfig}
\usepackage{amsmath, amsthm, amssymb, algorithm,thmtools,algpseudocode,enumerate,caption}

\usepackage[usenames,dvipsnames,svgnames,table]{xcolor}
\definecolor{darkgreen}{rgb}{0.0,0,0.9}

\usepackage[colorlinks=true,pdfpagemode=UseNone,citecolor=OliveGreen,linkcolor=BrickRed,urlcolor=BrickRed,
pagebackref]{hyperref}

\usepackage{tikz}

\newcommand{\Input}{\item[{\bf Input:}]}

\makeatletter
\newcommand{\setword}[2]{%
  \phantomsection
  #1\def\@currentlabel{\unexpanded{#1}}\label{#2}%
}
\makeatother

\def\R{\mathbb{R}}
\def\mP{\mathbb{P}}
\def\mE{\mathbb{E}}

\def\cM{{\cal M}}
\def\cE{{\cal E}}
\def\C{\mathbb{C}}
\def\eps{\epsilon}

\newcommand{\PP}[2]{\mP_{#1}\left[#2\right]}

\renewcommand{\P}[1]{\mP\left[#1\right]}

\newcommand{\EE}[2]{\mE_{#1}\left[#2\right]}
\newcommand{\norm}[1]{\|#1\|}
\newcommand{\supp}[1]{\textup{supp}\{#1\}}
\DeclareMathOperator{\TV}{TV}

\DeclareMathOperator{\poly}{poly}
\DeclareMathOperator{\image}{Im}
\DeclareMathOperator{\Var}{Var}

\numberwithin{equation}{section}

\newtheorem{theorem}{Theorem}[section]

\newtheorem{lemma}{Lemma}[section]
\newtheorem{fact}[theorem]{Fact}

\newtheorem{corollary}[theorem]{Corollary}
\newtheorem{definition}{Definition}[section]
\newenvironment{proofof}[1]{{\em Proof of #1.}}{\hfill
\qed}

\title{Monte Carlo Markov Chain Algorithms for Sampling \\Strongly Rayleigh Distributions and Determinantal Point Processes}

\author{
Nima Anari
\thanks{Computer Science Division, UC Berkeley.
Email: \protect\url{anari@berkeley.edu}}
\and
Shayan Oveis Gharan
\thanks{Department of Computer Science and Engineering, University of Washington.
Email: \protect\url{shayan@cs.washington.edu}}
\and
Alireza Rezaei
\thanks{Department of Computer Science and Engineering, University of Washington.
Email: \protect\url{arezaei@washington.edu}}
}

\begin{document}

\maketitle

\begin{abstract}
Strongly Rayleigh distributions are natural generalizations of product distributions and determinantal probability measures that satisfy the strongest form of negative dependence properties \cite{BBL09}.
We show that the ``natural'' Monte Carlo Markov Chain (MCMC) algorithm  mixes rapidly in the support of a {\em homogeneous} strongly Rayleigh distribution.
As a byproduct, our proof implies that Markov chains can be used to efficiently generate approximate samples of a $k$-determinantal point process. This answers an open question raised in \cite{DR10} which was studied recently in \cite{Kan13,LJS15,RK15}.
\end{abstract}


\section{Introduction}
Let $\mu:2^{[n]}\to \mathbb{R}_+$ be a probability distribution on the subsets of the set $[n]=\{1,2,\dots,n\}$. In particular, we assume that $\mu(.)$ is nonnegative and,
$$ \sum_{S\subseteq [n]} \mu(S) = 1.$$
 We assign a multi-affine polynomial with variables $z_1,\dots,z_n$ to $\mu$,
 $$ g_\mu(z) = \sum_{S\subseteq [n]} \mu(S)\cdot z^S,$$
 where for a set $S\subseteq[n]$, $z^S=\prod_{i\in S} z_i$.
The polynomial $g_\mu$ is also  known as the {\em generating polynomial} of $\mu$. We say $\mu$ is  {\em $k$-homogeneous} if $g_\mu$ is a homogeneous polynomial of degree $k$, i.e., if for any $S\in \supp{\mu}$, we have $|S|=k$.

A polynomial $p(z_1,\ldots,z_n)\in\C[z_1,\dots,z_n]$ is {\em stable}
if whenever $\image(z_i)>0$ for all $1\leq i\leq m$, $p(z_1,\ldots,z_m)\neq 0$. We say $p(.)$ is real stable, if it is stable and all of its coefficients are real. Real stable polynomials are considered to be a natural generalization of real rooted polynomials to multivariate polynomials. In particular, as a sanity check, it follows that any univariate polynomial is real stable  if and only if it is real rooted.
 We say that $\mu$ is a {\em strongly Rayleigh} distribution if $g_\mu$ is a real stable polynomial.

Strongly Rayleigh distributions are introduced and deeply studied in the  work of  \cite{BBL09}. 
These distributions are natural generalizations of  determinantal measures and random spanning tree distributions.  It is shown in \cite{BBL09} that strongly Rayleigh distributions satisfy the strongest form of negative dependence properties. 
These negative dependence properties were recently exploited to design approximation algorithms \cite{OSS11,PP14,AO14}.

In this paper we show that the ``natural'' Monte Carlo Markov Chain (MCMC) method  on the support of a homogeneous strongly Rayleigh distribution $\mu$ {\em mixes} rapidly. Therefore, this Markov Chain can be used to efficiently draw an approximate sample from $\mu$. Since determinantal point processes are special cases of strongly Rayleigh measures, our result implies that the same Markov chain efficiently generates random samples of a $k$-determinantal point process (see \autoref{sec:dpp} for the details).

We now describe the {\em lazy} MCMC $\cM_\mu$. The state space of $\cM$ is $\supp{\mu}$ and the transition probability kernel $P_\mu$ is defined as follows.
 We may drop the subscript if $\mu$ is clear in the context.
For a set $S\subseteq [n]$ and $i\in [n]$, let 
\begin{eqnarray*}
S-i &=& S\setminus \{i\},\\
S+i &=& S\cup \{i\}.	
\end{eqnarray*}
In a state $S$, choose an element $i\in S$ and $j\notin S$ uniformly and independently at random, and let $T=S-i+j$, then
\begin{enumerate}[i)]
\item If $T	\in \supp{\mu}$, move to $T$ with probability $\frac12 \min\{1, \mu(T)/\mu(S)\}$;
\item Otherwise, stay in $S$.
\end{enumerate}
It is easy to see that $\cM_\mu$ is reversible and $\mu(.)$ is the stationary distribution of the chain. 
In addition, Br\"and\'en showed that the support of a (homogeneous) strongly Rayleigh distribution is the set of bases of a matroid \cite[Cor 3.4]{Bra07}; so $\cM_\mu$ is irreducible. Lastly, since we stay in each state $S$ with probability at least $1/2$, $\cM_\mu$ is a lazy chain. 

In our main theorem we show that the above Markov chain is {\em rapidly mixing}. In particular, 
if we start $\cM_\mu$ from a state $S$, then after $O(n\cdot k\log(\frac{1}{\eps\cdot \mu(S)}))$ steps we obtain an $\eps$-approximate sample of the strongly Rayleigh distribution. 
First, we need to setup the notation. 
For probability distributions $\pi,\nu:\Omega\to\R_+$,
the total variation distance of $\pi,\nu$ is defined as follows:
$$ \norm{\nu-\pi}_{\TV} 
= \frac12\sum_{x\in\Omega} |\nu(x)-\pi(x)|.$$
If $X$ is a random variable sampled according to $\nu$ and $\norm{\nu - \pi}_{\TV}\leq \eps$, then we say $X$ is an $\eps$-approximate sample of $\pi$.
\begin{definition}[Mixing Time]
For a state $x\in\Omega$ and $\eps>0$, the total variation mixing time of a chain started at $x$ with transition probability matrix $P$  and stationary distribution $\pi$ is defined as follows:
$$ \tau_x(\eps):=\min\{t: \norm{P^t(x,.)-\pi}_{\TV}\leq\eps\},$$
where $P^t(x,.)$ is the distribution of the chain started at $x$ at time $t$.
\end{definition}

\noindent The following is our main theorem. 
\begin{theorem}\label{thm:SRmixing}
	For any strongly Rayleigh $k$-homogeneous probability distribution $\mu:2^{[n]}\to\R_+$, $S\in\supp{\mu}$ and $\eps>0$,
$$ \tau_S(\eps) \leq \frac1{C_\mu} \cdot \log\left(\frac{1}{\eps\cdot\mu(S)}\right),$$
where 
\begin{equation}\label{eq:Cmu}
	C_\mu:=\min_{S,T\in\supp{\mu}} \max(P_\mu(S,T),P_\mu(T,S))
\end{equation}
is at least $\frac1{2kn}$ by construction.
\end{theorem}
Suppose we have access to a set $S\in\supp{\mu}$ such that $\mu(S) \geq \exp(-n)$. In addition, we are given an oracle such that for any set $T\in {n\choose k}$, it returns $\mu(T)$ if $T\in\supp{\mu}$ and zero otherwise. Then, by the above theorem we can generate 
an $\eps$-approximate sample of $\mu$ with at most $\poly(n,k,\log(1/\eps))$ oracle calls.


For a strongly Rayleigh probability distribution $\mu:2^{[n]}\to\R_+$, and any integer $0\leq k\leq n$, the {\em truncation} of $\mu$ to $k$ is the conditional measure $\mu_k$ where for any $S\subseteq[n]$ of size $k$,
$$ \mu_k(S)=\frac{\mu(S)}{\sum_{S': |S'|=k} \mu(S')}.$$
Borcea, Br\"and\'en, and Liggett showed that for any strongly Rayleigh distribution $\mu$, and any integer $k$, $\mu_k$ is also strongly Rayleigh, \cite{BBL09}.
Therefore, if we have access to a set $S\subset [n]$ of size $k$, we can use the above theorem to generate an approximate sample of $\mu_k$.

\subsection{Determinantal Point Processes and the Volume Sampling Problem}
\label{sec:dpp}
A determinantal point process (DPP) on a set of elements $[n]$ is a probability distribution $\mu$ on the set $2^{[n]}$ identified by a positive semidefinite {\em ensemble} matrix $L\in\R^{n\times n}$ where for any $S\subseteq[n]$ we have 
$$ \P{S} \propto \det(L_S),$$
where $L_S$ is the principal submatrix of $L$ indexed by the elements of $S$. 

DPPs are one of the fundamental objects used to study a variety of tasks in machine learning, including text summarization, image search, news threading, etc. For more information about DPPs and their applications we refer to a recent survey by Kulesza and Taskar \cite{KT13}.

For an integer $0\leq k\leq n$, and a DPP $\mu$, the truncation of $\mu$ to $k$, $\mu_k$ is called a $k$-DPP. It turns out that the family of determinantal point processes are not closed under truncation. Perhaps, the simplest example is the $k$-uniform distribution over a set of $n$ elements. Although the uniform distribution over $n$ elements is a DPP, for any $2\leq k\leq n-2$, the corresponding $k$-DPP is not a DPP \cite[Section 5]{KT13}.

Generating a sample from a $k$-DPP is a fundamental computational task with many practical applications \cite{KV09,DR10,KT13}.
This problem is also equivalent to the $k$-volume sampling problem \cite{DRVW06,KV09,BMD09,DR10,GS12a} which has applications in low-rank approximation and row-subset selection problem.  In the $k$-volume sampling problem, we are given a matrix $X\in\R^{n\times m}$ and we want to choose a set $S\subseteq [n]$ of $k$ rows of $X$ with probability proportional to $\det(X_{S,[m]},X_{S,[m]}^\intercal)$, where $X_{S,[m]}\in\R^{k\times m}$ is the submatrix of $X$ with rows indexed by elements of $S$. If $L$ is the ensemble matrix of a given $k$-DPP $\mu$, and $L=XX^\intercal$ is the Cholesky decomposition of $L$, then the $k$-volume sampling problem on $X$ is equivalent to the problem of generating a random sample of~$\mu$.

 In the past, several spectral algorithms were designed for 
sampling from $k$-DPPs \cite{HKPV06,DR10,KT13}, but these 
algorithms typically need to diagonalize a giant $n$-by-$n$ 
matrix, so they are  inefficient in time and memory
\footnote{We remark that the algorithms in \cite{DR10} are 
almost linear in $n$; however they need  access to the 
Cholesky decomposition of the  ensemble matrix of the 
underlying DPP.}. It was asked by  \cite{DR10} to generate random samples of a $k$-DPP using Markov chain techniques. Markov chain techniques are very appealing in this context because 
of their simplicity and efficiency. 
There has been several attempts \cite{Kan13, LJS15, RK15} to upper bound the mixing time of the Markov chain $\cM_\mu$ for a $k$-DPP $\mu$; but, to the best of our knowledge, this question is still open\footnote{We remark that \cite{Kan13} claimed to have a proof of the rapid mixing time of a similar Markov chain. As it is pointed out in \cite{RK15} the coupling argument of \cite{Kan13} is ill-defined. To be more precise, the chain specified in Algorithm 1 of \cite{Kan13} may not mix in a polynomial time of $n$. The chain specified in Algorithm 2 of \cite{Kan13} is similar to $\cM_\mu$, but the statement of Theorem 2 which upper bounds its mixing time is clearly incorrect even when $k=1$.}.

Here, we show that for a $k$-DPP $\mu$, $\cM_\mu$ can be used to efficiently generate an approximate sample of $\mu$.
\cite{BBL09} show that any DPP is a strongly Rayleigh distribution. Since strongly Rayleigh distributions are closed under truncation, any $k$-DPP is a strongly Rayleigh distribution. Therefore, by \autoref{thm:SRmixing}, for any $k$-DPP $\mu$, $\cM_{\mu}$ mixes rapidly to the stationary distribution. 
\begin{corollary}
\label{cor:kDPPmixing}
For any $k$-DPP $\mu:2^{[n]}\to\R_+$,  $S\in\supp{\mu}$ and $\eps>0$,
$$ \tau_S(\eps) \leq \frac1{C_\mu}\cdot \log\left(\frac{1}{\eps\cdot\mu(S)}\right).$$
\end{corollary}

Given access to the ensemble matrix of a $k$-DPP, we can use the above theorem to generate $\eps$-approximate samples of the $k$-DPP. 

\begin{theorem}
Given an ensemble matrix $L$ of a $k$-DPP $\mu$ there is an algorithm for any $\eps>0$ generates an $\eps$-approximate sample of $\mu$ in time $\poly(k) O(n\log(n/\eps))$. 
\end{theorem}
To prove the above theorem, we need an efficient algorithm to generate a set $S\in\supp{\mu}$ such that $\mu(S)$ is bounded away from zero, perhaps by an exponentially small function of $n,k$.
We use the greedy algorithm \ref{alg:greedy} to find such a set, and we show that, in time $O(n)\poly(k)$, it returns a set $S$ such that
\begin{equation}\label{eq:largevolS} \det(L_S) \geq \frac{1}{k!\cdot |\supp{\mu}|} \geq \frac{1}{k! \cdot {n\choose k}}\geq n^{-k}.	
\end{equation}
Noting that each transition step of the Markov chain $\cM_\mu$ only takes time that is polynomial in $k$, this completes the proof of the above theorem.

\begin{algorithm}
\begin{algorithmic}
\Input The ensemble matrix, $L$, of a $k$-DPP $\mu$.
\State $S\leftarrow \emptyset$.
\For {$i=1\to k$}
\State	Among all elements $j\notin S$ add the one maximizng
	$ \det(L_{S+j}).$
\EndFor
\end{algorithmic}
\caption{Greedy Algorithm for Selecting the Start State of $\cM_\mu$}
\label{alg:greedy}	
\end{algorithm}

It remains to analyze \autoref{alg:greedy}. This  is already done in  \cite{CM09} in the context of  maximum volume submatrix problem.  In the maximum volume submatrix problem, we are given a matrix $X\in\R^{n\times m}$ and we want to choose a subset $S$ of $k$ rows of $X$ maximizing $\det(X_{S,[m]}X_{S,[m]}^\intercal)$. Equivalently, given a matrix $L=XX^\intercal$, and we want to choose $S\subseteq [n]$ of size $k$ maximizing $\det(L_S)$. Note that if $L$ is an ensemble matrix of a $k$-DPP $\mu$, then
$$ \max_{|S|=k} \det(L_S) \geq \frac{1}{|\supp{\mu}|}.$$
The maximum volume submatrix problem is NP-hard to approximate within a factor $c^k$ for some constanat $c>1$ \cite{CM13}. Numerous approximation algorithm is given for this problem \cite{CM09, CM13, Nik15}.
It was shown in \cite[Thm 11]{CM09} that choosing the rows of $X$ greedily gives a $k!$ approximation to the maximum volume submatrix problem. Algorithm \ref{alg:greedy} is equivalent to the greedy algorithm of \cite{CM09}; it is only described in the language of ensemble matrix $L$.
Therefore, it returns a set $S$ such that
$$ \det(L_S) \geq \frac{\max_{|T|=k}\det(L_T)}{k!} \geq \frac{1}{k!|\supp{\mu}|}, $$
as desired.

\subsection{Proof Overview}
In the rest of the paper we prove \autoref{thm:SRmixing}. To prove \autoref{thm:SRmixing} we lower bound the spectral gap, a.k.a. the Poincar\'e constant of the chain $\cM_\mu$. This directly upper bounds the mixing time in total variation distance.
To lower bound the spectral gap, we use an extension of the seminal work of \cite{FM92}. Feder and Mihail showed that the {\em bases exchange} graph of the bases of a {\em balanced matroid} is an {\em expander}. This directly lower bounds the spectral gap by Cheeger's inequality. 
A matroid is called balanced if the matroid and all of its minors satisfy the property that, the uniform  distribution of the bases is negatively associated  (see \autoref{sec:strongrayleigh} for the definition of negative association). 

Our proof can be seen as a {\em weighted} variant of \cite{FM92}. 
As we mentioned earlier, the support of a homogeneous strongly Rayleigh distribution corresponds to the bases of a matroid.
Our proof shows that if a distribution $\mu$ over the bases of a matroid and all of its conditional measures are negatively associated,  then the MCMC algorithm mixes rapidly. 
To show that $\mu$ satisfies the aforementioned property we simply appeal to the negative dependence theory of strongly Rayleigh distributions developed in \cite{BBL09}.
Although our proof can be written in the language of \cite{FM92},  we work with the more advanced chain decomposition idea of \cite{JSTV04} to prove a tight bound on the Poincar\'e constant, see \autoref{sec:decomposition} for the details.

We remark that, in general, the decomposition idea of \cite{JSTV04} can be also be used to lower bound the {\em log-Sobolev} constant. However, it turns out that in our case, the log-Sobolev constant may  be no larger than $\frac{1}{-\log(\min_{S\in\supp{\mu}} \mu(S))}$. Since the latter quantity is not necessarily lower-bounded by a function of $k$ and $n$, the log-Sobolev constant (and hence, the $L_2$ mixing time) of the chain may be unbounded.

\section{Background}
\subsection{Markov Chains and Mixing Time}
\label{sec:mixingtime}
In this section we give a high level overview of Markov chains and their mixing times. We refer the readers for \cite{LPW06,MT06} for details.
Let $\Omega$ denote the state space, $P$ denote the Markov kernel and $\pi(.)$ denote the stationary distribution of a Markov chain.
We say a Markov chain is {\em lazy} if for any state $x\in\Omega$, $P(x,x)\geq 1/2$.

A Markov chain $(\Omega, P, \pi)$ is reversible if for any pair of states $x,y\in \Omega$, $\pi(x)P(x,y) = \pi(y)P(y,x)$. This is known as the {\em detailed balanced} condition. In this paper we only work with reversible chains. We equip the space of all functions $f:\Omega\to\R$ with the standard inner product for $L^2(\pi)$,
$$ \langle f,g\rangle_\pi := \EE{\pi}{f\cdot g}=\sum_{x\in \Omega} \pi(x) f(x) g(x).$$
In particular, $\norm{f}_\pi=\sqrt{\langle f,f\rangle_\pi}$.
For a function $f\in L^2(\pi)$, the {\em Dirichlet form} $\cE_\pi(f,f)$ is defined as follows
	$$ \cE_\pi(f,f):=\frac12 \sum_{x,y\in\Omega} (f(x)-f(y))^2 P(x,y)\pi(x), $$
and the {\em Variance} of $f$ is
$$ \Var_\pi(f) := \norm{f - \mE_\pi f}_\pi^2 = \sum_{x\in\Omega} (f(x) - \mE_\pi f)^2\pi(x). $$


Next, we overview classical spectral techniques to upper bound the mixing time of Markov chains.

\begin{definition}[Poincar\'e Constant]
The {\em Poincar\'e constant} of the chain is defined as follows,
$$ \lambda:=\inf_{f:\Omega\to\R} \frac{\cE_\pi(f,f)}{\Var_\pi(f)},$$
where the infimum is over all functions with nonzero variance.
\end{definition}
It is easy to see that for any transition probability matrix $P$, the second largest eigenvalue of $P$ is $1-\lambda$.
If $P$ is a lazy chain, then $1-\lambda$ is also the second largest eigenvalue of $P$ in absolute value.
In the following fact we see how to calculate the Poincar\'e constant of any reversible 2-state chain.
\begin{fact}\label{fact:2statechain}
The Poincar\'e constant of any reversible two state chain with $\Omega={0,1}$ and $P(0,1)=c\cdot \pi(1)$ is $c$.
\end{fact}
\begin{proof}
Consider any function $f$. Since $\Var(f)$ is shift-invariant, we can assume $\mE_\pi f=0$, i.e., $\pi(0)f(0)=-\pi(1)f(1)$. Since $\frac{\cE_\pi(f,f)}{\Var_\pi(f)}$ is invariant under the scaling of $f$, we can assume $f(0)=\pi(1)$ and $f(1)=-\pi(0)$. 
Since the chain is reversible $P(1,0)=c\cdot \pi(0)$.
Plugging this unique $f$ into the ratio we obtain $\lambda=c$.
\end{proof}
To prove \autoref{thm:SRmixing} we simply calculate the Poincar\'e constant of the chain $\cM_\mu$ and then we use  the following classical theorem of Diaconis and Stroock to upper bound the mixing time. 

\begin{theorem}[{\cite[Prop 3]{DS91}}]\label{thm:mixingtime}
For any reversible irreducible lazy Markov chain $(\Omega, P, \pi)$ with Poincar\'e constant $\lambda$, $\eps>0$ and any state $x\in \Omega$,
$$ \tau_x(\eps) \leq \frac1{\lambda}\cdot \log\left(\frac{1}{\eps\cdot \pi(x)}\right)$$	
\end{theorem}

Using the above theorem, to prove \autoref{thm:SRmixing}, it is enough to lower bound the Poincar\'e constant of $\cM_\mu$.
\begin{theorem}\label{thm:SRgap}
	For any $k$-homogeneous strongly Rayleigh distribution $\mu:2^{[n]}\to\R_+$, the Poincar\'e constant of the chain $\cM_\mu=(\Omega_\mu,P_\mu,\mu)$ is at least
	$$ \lambda \geq C_\mu.$$
\end{theorem}
It is easy to see that \autoref{thm:SRmixing} follows by the above two theorems.

\subsection{Strongly Rayleigh Measures}
\label{sec:strongrayleigh}
A probability distribution $\mu:2^{[n]}\to\R_+$ is pairwise {\em negatively correlated} if for any pair of elements $i,j\in [n]$,
$$ \PP{S\sim \mu}{i\in S}\cdot \PP{S\sim\mu}{j\in S} \geq \PP{S\sim\mu}{i,j\in S}.$$
Feder and Mihail \cite{FM92} defined {\em negative association} as a generalization of negative correlation. 
We say an event ${\cal A}\subseteq 2^{[n]}$ is increasing if it is closed upward under containment, i.e., if $S\in{\cal A}$, and $S\subseteq T$, then $T\in {\cal A}$. 
We say a function $f:2^{[n]}\to\R_+$ is {\em increasing} if it is the indicator function of an increasing event.
We say $\mu$ is negatively associated if for any pair of increasing functions $f,g:2^{[n]}\to\R_+$ depending on disjoint sets of coordinates,
$$ \EE{\mu}{f}\cdot\EE{\mu}{g} \geq \EE{\mu}{f\cdot g}.$$ 
Building on \cite{FM92}, Borcea, Br\"and\'en and Liggett proved that any strongly Rayleigh distribution is negatively associated. 
\begin{theorem}[\cite{BBL09}]
\label{thm:negativeass}
	Any strongly Rayleigh probability distribution is negatively associated. 
\end{theorem}
As an example, the above theorem implies that any $k$-DPP is negatively associated. The negative association property is the key to our lower bound on the Poincar\'e constant of the chain $\cM_\mu$.

For  $1\leq i\leq n$, let $Y_i$ be the random variable indicating whether $i$ is in a sample of $\mu$. We use 
$$\mu|_{i}:=\{\mu| Y_i=1\},$$
to denote the conditional measure on sets that contain $i$ and
$$ \mu|_{\overline{i}}:=\{\mu | Y_i=0\}, $$
to denote the conditional measure on sets that do not contain $i$.
Borcea, Br\"and\'en and Ligett showed that strongly Rayleigh distributions are closed under conditioning.
\begin{theorem}[\cite{BBL09}]\label{thm:SRconditioning}
For any strongly Rayleigh distribution $\mu:2^{[n]}\to\R_+$ and any $1\leq i\leq n$, $\mu|_{\overline{i}},\mu|_{i}$  are strongly Rayleigh. 
\end{theorem}
The above two theorems are the only properties of the strongly Rayleigh distributions that we use in the proof of \autoref{thm:SRmixing}. In other words, the statement of \autoref{thm:SRmixing} holds for any homogeneous probability distribution $\mu:2^{[n]}\to\R_+$ where $\mu$ and all of its conditional measures are negatively associated.

\subsection{Decomposable Markov Chains}
\label{sec:decomposition}
In this section we describe the decomposable Markov chain technique due to Jerrum, Son, Tetali and Vigoda \cite{JSTV04}. This will be our main tool to lower bound the Poincar\'e constant of $\cM_\mu$. Roughly speaking, they consider Markov chains that can be decomposed into ``projection'' and ``restriction'' chains. They lower bound the Poincar\'e constant of the original chain assuming certain properties of these projection/restriction chains. 

Let $\Omega_0 \cup \Omega_1$ be a decomposition of the state space of a Markov chain $(\Omega, P, \pi)$ into two disjoint sets\footnote{Here, we only focus on decomposition into two disjoint sets, although the technique of \cite{JSTV04} is more general.}. For $i\in\{0,1\}$ let
$$ \bar{\pi}(i) = \sum_{x\in \Omega_i} \pi(x),$$
and let $\bar{P}\in\R^{2\times 2}$ be
$$ \bar{P}(i,j) = \bar{\pi}(i)^{-1} \sum_{x\in \Omega_i,y\in\Omega_j} \pi(x) P(x,y).$$ 
The Markov chain $(\{0,1\},\bar{P},\bar{\pi})$ is called a projection chain. 
Let $\bar{\lambda}$ be the Poincar\'e constant of this chain.

We can also define a restriction Markov chain on each $\Omega_i$ as follows. For each $i\in \{0,1\}$,
$$ P_i(x,y) = \begin{cases}
 P(x,y) & \text{if } x\neq y,\\
 1-P(x,x)-\sum_{z\notin \Omega_i} P(x,z) 
 & \text{if }x=y.	
 \end{cases}
$$
In other words, for any transition from $x$ to a state outside of $\Omega_i$, we remain in $x$. Observe that in the stationary distribution of the restriction chain, the probability of $x$  is proportional to $\pi(x)$.
Let $\lambda_i$ be the Poincar\'e constant of the chain $(\Omega_i, P_i, .)$. 
Now, we are ready to explain the main result of \cite{JSTV04}.
\begin{theorem}[{\cite[Cor 3]{JSTV04}}]\label{thm:decompositionchain}
If for any distinct $i,j\in \{0,1\}$, and any $x\in \Omega_i$, 
\begin{equation}  
 \label{eq:Phatuniform}
\bar{P}(i,j) = \sum_{y\in\Omega_j} P(x,y),
\end{equation} 
then the Poincar\'e constant of $(\Omega,P,\pi)$ is at least $\min\{\bar{\lambda},\lambda_0,\lambda_1\}$.	
\end{theorem}


\section{Inductive Argument}
In this section we prove \autoref{thm:SRgap}.
Throughout this section we fix a strongly Rayleigh distribution $\mu$, and we let $\Omega,P$ be the state space and the transition probability matrix of $\cM_\mu$.

We prove \autoref{thm:SRgap} by induction on $|\supp{\mu}|$. 
If $|\supp{\mu}|=1$, then there is nothing to prove. 
 To do the induction step, we will use \autoref{thm:decompositionchain}. So, let us first start by defining the restriction chains. 
Without loss of generality, perhaps after renaming, let $n$ be an element such that $0<\PP{S\sim\mu}{n\in S}<1$.
Let $\Omega_0=\{S\in\supp{\mu}: n\notin S\}$ and $\Omega_1=\{S\in\supp{\mu}: n\in S\}$. Note that both of these sets are nonempty. Observe that the restricted chain $(\Omega_0, P_0, .)$ is the same as $\cM_{\mu|_{\overline{n}}}$ and $(\Omega_1,P_1,.)$ is the same as $\cM_{\mu|_{n}}$.
In addition, by \autoref{thm:SRconditioning}, $\cM_{\mu|_{\overline{n}}}$ and $\cM_{\mu|_{n}}$ are strongly Rayleigh, and also clearly
$C_{\mu|_n},C_{\mu|_{\overline{n}}} \geq C_\mu$.
So, we can use the induction hypothesis to lower bound 
$\lambda_0, \lambda_1 \geq C_\mu$.

 It remains to lower bound the Poincar\'e constant of the projection chain and  to prove equation \eqref{eq:Phatuniform}.
Unfortunately, $P$ does not satisfy \eqref{eq:Phatuniform}.
So, we use an idea of \cite{JSTV04}. We construct a new Markov kernel $\hat{P}$ such that 
(i) $\hat{P}$ has the same stationary distribution $\mu$.
(ii) The Poincar\'e constant of $\hat{P}$, $\hat{\lambda}$, lower-bounds $\lambda$. 
Then, we use \autoref{thm:decompositionchain} to lower bound $\hat{\lambda}$.

To make sure that $\hat{P}$ satisfies (i), (ii), it is enough that for all distinct states $x,y\in\Omega$, 
\begin{eqnarray} 
\mu(x)\hat{P}(x,y)&=&\mu(y)\hat{P}(y,x),\label{eq:hPmu}	\\
\hat{P}(x,y) &\leq& P(x,y).\label{eq:hPsP}
\end{eqnarray}
Equation \eqref{eq:hPmu} implies (i), i.e., that $\mu$ is also the stationary distribution of $\hat{P}$. By an application of the comparison method \cite{DS93}, (i) together with \eqref{eq:hPsP} implies (ii), i.e.,
\begin{equation}\label{eq:lambdabigger}\hat{\lambda} \leq \lambda.	
\end{equation}
So, to prove the induction step, it is enough to show that
\begin{equation} \hat{\lambda} \geq C_\mu.
	\label{eq:hlambdalower}
\end{equation}
\begin{lemma}\label{lem:hatP}
There is a transition probability matrix $\hat{P}:\Omega\times \Omega\to\R_+$ such that 
\begin{enumerate}[1)]
\item	$\hat{P}$ satisfies  \eqref{eq:hPsP}, \eqref{eq:hPmu}. 
\item For any $i\in\{0,1\}$ and states $x,y\in\Omega_i$, $\hat{P}(x,y)=P(x,y)$.
\item The Poincar\'e constant of the chain $(\Omega,\hat{P},\mu)$ projected onto $\Omega_0,\Omega_1$ is at least $\bar{\hat{\lambda}}\geq C_\mu$,
\item For any state $S\in \supp{\mu}$ and distinct $i,j\in\{0,1\}$,
$$ \bar{\hat{P}}(i,j) = \sum_{y\in\Omega_j} \hat{P}(x,y).$$
\end{enumerate}
\end{lemma}
Before, proving the above lemma, we use it to finish the proof of the induction.
By part (2), $\hat{P}$ agrees with $P$ on the projection chains. Therefore, the Poincar\'e constants of the chains $(\Omega_0,\hat{P}_0,.)$ and $(\Omega_1,\hat{P}_1,.)$ are at least $\lambda_0, \lambda_1\geq C_\mu$.
So, by parts (3) and (4) we can invoke \autoref{thm:decompositionchain} for $\hat{P}$ and we get that
$$ \hat{\lambda} \geq \min\{\bar{\hat{\lambda}},\hat{\lambda}_0,\hat{\lambda}_1\} \geq C_\mu. $$
This proves \eqref{eq:hlambdalower}. As we discussed earlier, part (1) implies \eqref{eq:lambdabigger} which completes the induction.

\subsection{Proof of \autoref{lem:hatP}}
In the rest of this section we prove \autoref{lem:hatP}.
Note that the main challenge in proving the lemma is part (4). The transition probability matrix $P$ already satisfies parts (1)-(3).
The key to prove part (4) is to construct a fractional perfect matching between the states of $\Omega_0$ and $\Omega_1$, see the following lemma for the formal definition. This idea originally was used in \cite{FM92} and it was later extended in \cite{JS02}. 
\begin{lemma}\label{lem:fractionalmatching}
There is a function $w:\{\{x,y\}: x\in \Omega_0,y\in\Omega_1\} \to \R_+$ such that $w_{\{x,y\}} > 0$ only if $P(x,y) > 0$ and
\begin{equation}
\label{eq:flow}
\begin{aligned}
\sum_{y \in \Omega_1} w_{\{x,y\}} &= \frac{\mu(x)}{\mu(\Omega_0)} \, & \forall x \in \Omega_0,\\
\sum_{x \in \Omega_0} w_{\{x,y\}} &= \frac{\mu(y)}{\mu(\Omega_1)} \, & \forall y \in \Omega_1. 
\end{aligned}
\end{equation}
\end{lemma}
We use the negative association property of the strongly Rayleigh distributions to prove the above lemma. But before that let us prove \autoref{lem:hatP}.

\begin{proofof}{\autoref{lem:hatP}}
We use $w$ to construct $\hat{P}$. 
For any $i,j\in\{0,1\}$ and $x \in \Omega_i$ and $y \in \Omega_j$, we let
\begin{equation*}
\hat{P}(x,y)= \begin{cases}
\frac{C_\mu}{\mu(x)} \mu(\Omega_i)\mu(\Omega_j)w_{\{x,y\}} & \text{if } i\neq j,\\
P(x,y) & \text{otherwise}. 
\end{cases}
\end{equation*}
Note that by definition part (2) is satisfied. First we verify part (1). If $i\neq j$, then
$$ \hat{P}(x,y)\mu(x) = C_\mu \mu(\Omega_i)\mu(\Omega_j)w_{\{x,y\}} =  \hat{P}(y,x)\mu(y), $$
and if $i=j$ the same identity holds because $\hat{P}(x,y)=P(x,y)$.
This proves \eqref{eq:hPmu}. To see \eqref{eq:hPsP} note that for any $i\neq j$ and $x\in\Omega_i,y\in\Omega_j$ we have
\begin{eqnarray*}
\hat{P}(x,y)&=& 
\frac{C_\mu}{\mu(x)}\cdot\mu(\Omega_i)\mu(\Omega_j)w_{\{x,y\}}\\
&\leq& \frac{\max(P(x,y),P(y,x))}{\mu(x)}\mu(\Omega_i)\mu(\Omega_j)w_{\{x,y\}}\\
&\leq & \max(P(x,y),P(y,x))\cdot \frac{\min(\mu(x),\mu(y))}{\mu(x)} \leq P(x,y).
\end{eqnarray*}
The first inequality follows by the definition of $C_\mu$ (see \eqref{eq:Cmu}), the second inequality follows by the fact that  $w_{\{x,y\}}\leq \frac{\mu(x)}{\mu(\Omega_i)}$ and $w_{\{x,y\}} \leq \frac{\mu(y)}{\mu(\Omega_j)}$, and the last inequality follows by the detailed balanced condition.
This completes the proof of part (1).

Next, we prove part (3).  
By definition of $\hat{P}$, for distinct $i,j\in\{0,1\}$ we have
\begin{eqnarray*}\bar{\hat{P}}(i,j) &=& \frac{1}{\mu(\Omega_i)} \sum_{x\in \Omega_i,y\in\Omega_j} \mu(x) \hat{P}(x,y) \\
&=&
\frac{C_\mu}{\mu(\Omega_i)}\sum_{x\in 
\Omega_i,y\in \Omega_j} \mu(\Omega_i)
\mu(\Omega_j)w(x,y)\\
&=& C_{\mu}\cdot \mu(\Omega_j)\sum_{x \in \Omega_i} \frac{\mu(x)}{\mu(\Omega_i)} 
=C_\mu\cdot \mu(\Omega_j),
\end{eqnarray*}
where the second to last equality follows by \eqref{eq:flow}.
By \autoref{fact:2statechain}, the Poincar\'e constant of $\bar{\hat{P}}=C_\mu$. This proves part (3).

Finally we prove part (4). Fix distinct $i,j\in\{0,1\}$ and $z\in\Omega_i$. We have,
$$ \sum_{y\in\Omega_j} \hat{P}(z,y) = \frac{C_\mu}{\mu(z)} \mu(\Omega_i)\mu(\Omega_j)w_{\{z,y\}} =C_\mu \cdot \mu(\Omega_j), $$
where we used \eqref{eq:flow}.
On the other hand,
by definition of $\hat{P}$ we know that 
\begin{eqnarray*}
\bar{\hat{P}}(i,j) = \frac{1}{\mu(\Omega_i)}\sum_{x \in \Omega_i,y\in\Omega_j} \mu(x)\bar{\hat{P}}(x,y)= 
C_\mu\cdot  \mu(\Omega_j) \sum_{x\in\Omega_i} \frac{\mu(x)}{\mu(\Omega_i)} = C_\mu \cdot \mu(\Omega_j),
\end{eqnarray*}
where the second equality follows by \eqref{eq:flow}.
This completes the proof of part (4) and \autoref{lem:hatP}.
\end{proofof}

It remains to prove \autoref{lem:fractionalmatching}.
For a set $A\subseteq \Omega$ let
$$N(A)= \{ y \in \Omega\setminus A : \exists x \in A, P(x,y) >0 \,\}.$$
To prove \autoref{lem:fractionalmatching} we use a maximum flow-minimum cut argument. To prove the claim we need to show that the support graph of the transition probability matrix $P_\mu$ satisfies Hall's condition. This is proved in the following lemma using the negative association property of strongly Rayleigh measures. The proof is simply an extension of the proof of \cite[Lem 3.1]{FM92}.
\begin{lemma}
\label{lem:enforcementratio}
For any $A \subseteq \Omega_1$, 
$$\frac{\mu(N(A))}{\mu(\Omega_0)} \geq \frac{\mu(A)}{\mu(\Omega_1)}.$$
\end{lemma}
\begin{proof}
Let $R\sim\mu$ be a random set. 
Recall that $\Omega_0 = \{ S \in \supp{u}: n \notin S\}$ and $\Omega_1= \{ S \in \supp{u} : n \in S\}$. 
Let $g$ be a random variable indicating whether $n \in R$.
Let $f$ be a indicator random  variable  which is $1$ if there exists  $T\in A$ such that $R\supseteq T \setminus \{n\}$.  It is easy  to see that $f$ and $g$ are two increasing functions which are supported on two disjoint sets of elements. By the negative association property, \autoref{thm:negativeass}, we can write
$$\PP{\mu}{f(R)=1 | g(R)=0} \geq \PP{\mu}{f(R)=1|g(R)=1}.$$
The lemma follows by the fact that the LHS of the above inequality is $\frac{\mu(N(A))}{\mu(\Omega_0)}$ and the RHS is $\frac{\mu(A)}{\mu(\Omega_1)}$.
\end{proof}

\begin{proofof}{\autoref{lem:fractionalmatching}}
Let $G$ be a bipartite graph on $\Omega_0 \cup \Omega_1$ 
where there is an edge between $x\in \Omega_1$ and $y \in 
\Omega_0$ if $P(x,y)>0$. We prove the lemma by showing there 
is a unit flow from $\Omega_1$ to $\Omega_0$ such that the 
amount of the flow going out of any $x \in \Omega_1$ is $
\frac{\mu(x)}{\mu(\Omega_1)}$, and the incoming flow to any $y \in 
\Omega_0$ is $\frac{\mu(y)}
{\mu(\Omega_0)}$. Then, we simply let $w_{\{x,y\}}$ be the flow on the edge connecting $x$ to $y$.

 Add a source $s$ and a sink $t$. For any $x\in \Omega_1$ add an arc $(s,x)$ with capacity $c_{s,x}=\mu(x)/\mu(\Omega_1)$. Similarly, for any $y \in \Omega_0$ add an arc $(y,t)$ with capacity
$c_{y,t}=\mu(y)/\mu(\Omega_0)$. Let the capacity of any other edge in the graph be $\infty$. 
Since the sum of the capacity of all edges leaving $S$ is 1, to prove the lemma, it is enough to show that the maximum flow is 1.
Equivalently, by the max-flow min-cut theorem, it suffices to show the value of the minimum cut separating $s$ and $t$ is at least $1$. Let $B,\overline{B}$ be an arbitrary $s$-$t$ cut, i.e.,  $s \in B$ and $t \in \overline{B}$. 
Let $B_0 = \Omega_0 \cap B$ and $B_1 = \Omega_1 \cap B$. For $X \subseteq \Omega_1$, $Y \subseteq \Omega_0$, let 
$c(X,Y)= \sum_{x \in X , y \in Y} c_{x,y}$. We have
\begin{eqnarray}
c(B,\overline{B}) &\geq& c(s,\Omega_1\setminus{B_1})+ c(B_0,t) \nonumber\\
&=& \frac{\mu(\Omega_1\setminus{B_1})}{\mu(\Omega_1)}+ \frac{\mu(B_0)}{\mu(\Omega_0)} \nonumber\\
&=& 1-\frac{\mu(B_1)}{\mu(\Omega_1)} +\frac{\mu(B_0)}{\mu(\Omega_0)} \geq 1-\frac{\mu(N(B_1))}{\mu(\Omega_0)} + \frac{\mu(B_0)}{\mu(\Omega_0)},
\end{eqnarray}
where the inequality follows by \autoref{lem:fractionalmatching}.
If there is any edge from $B_1$ to $\Omega_0\setminus B_0$, then $c(B,\overline{B})=\infty$ and we are done. Otherwise, $N(B_1) \subseteq B_0$. Therefore,  $\mu(N(B_1)) \leq \mu(B_0)$, and the RHS of the above inequality is at least 1. So, $c(B,\overline{B})\geq 1$ as desired.
\end{proofof}

\bibliographystyle{alpha}
\bibliography{ref2}

\begin{thebibliography}{DRVW06}

\bibitem[AO15]{AO14}
Nima Anari and Shayan {Oveis Gharan}.
\newblock {{Effective-Resistance-Reducing Flows and Asymmetric TSP}}.
\newblock In {\em FOCS}, pages 20--39, 2015.

\bibitem[BBL09]{BBL09}
Julius Borcea, Petter Branden, and Thomas~M. Liggett.
\newblock {Negative dependence and the geometry of polynomials.}
\newblock {\em Journal of American Mathematical Society}, 22:521--567, 2009.

\bibitem[BMD09]{BMD09}
Christos Boutsidis, Michael~W Mahoney, and Petros Drineas.
\newblock An improved approximation algorithm for the column subset selection
  problem.
\newblock In {\em SODA}, pages 968--977, 2009.

\bibitem[Br{\"a}07]{Bra07}
Petter Br{\"a}nd{\'e}n.
\newblock {Polynomials with the half-plane property and matroid theory}.
\newblock {\em Advances in Mathematics}, 216(1):302--320, 2007.

\bibitem[{\c{C}}MI09]{CM09}
Ali {\c{C}}ivril and Malik Magdon-Ismail.
\newblock On selecting a maximum volume sub-matrix of a matrix and related
  problems.
\newblock {\em Theoretical Computer Science}, 410(47):4801--4811, 2009.

\bibitem[{\c{C}}MI13]{CM13}
Ali {\c{C}}ivril and Malik Magdon-Ismail.
\newblock Exponential inapproximability of selecting a maximum volume
  sub-matrix.
\newblock {\em Algorithmica}, 65(1):159--176, 2013.

\bibitem[DR10]{DR10}
Amit Deshpande and Luis Rademacher.
\newblock Efficient volume sampling for row/column subset selection.
\newblock In {\em FOCS}, pages 329--338. IEEE, 2010.

\bibitem[DRVW06]{DRVW06}
Amit Deshpande, Luis Rademacher, Santosh Vempala, and Grant Wang.
\newblock Matrix approximation and projective clustering via volume sampling.
\newblock In {\em SODA}, pages 1117--1126, 2006.

\bibitem[DS91]{DS91}
Persi Diaconis and Daniel Stroock.
\newblock Geometric bounds for eigenvalues of markov chains.
\newblock {\em The Annals of Applied Probability}, pages 36--61, 1991.

\bibitem[DSC93]{DS93}
Persi Diaconis and Laurent Saloff-Coste.
\newblock Comparison theorems for reversible markov chains.
\newblock {\em The Annals of Applied Probability}, pages 696--730, 1993.

\bibitem[FM92]{FM92}
Tom\'{a}s Feder and Milena Mihail.
\newblock {Balanced matroids}.
\newblock In {\em Proceedings of the twenty-fourth annual ACM symposium on
  Theory of Computing}, pages 26--38, New York, NY, USA, 1992. ACM.

\bibitem[GS12]{GS12a}
Venkatesan Guruswami and Ali~Kemal Sinop.
\newblock Optimal column-based low-rank matrix reconstruction.
\newblock In {\em SODA}, pages 1207--1214. SIAM, 2012.

\bibitem[HKPV06]{HKPV06}
J.B. Hough, M.~Krishnapur, Y.~Peres, and B.~Vir\'ag.
\newblock Determinantal processes and independence.
\newblock {\em Probability Surveys}, (3):206--229, 2006.

\bibitem[JS02]{JS02}
Mark Jerrum and Jung~Bae Son.
\newblock Spectral gap and log-sobolev constant for balanced matroids.
\newblock In {\em FOCS}, pages 721--729, 2002.

\bibitem[JSTV04]{JSTV04}
Mark Jerrum, Jung-Bae Son, Prasad Tetali, and Eric Vigoda.
\newblock Elementary bounds on poincar{\'e} and log-sobolev constants for
  decomposable markov chains.
\newblock {\em Annals of Applied Probability}, pages 1741--1765, 2004.

\bibitem[Kan13]{Kan13}
Byungkon Kang.
\newblock Fast determinantal point process sampling with application to
  clustering.
\newblock In {\em NIPS}, pages 2319--2327, 2013.

\bibitem[KT13]{KT13}
Alex Kulesza and Ben Taskar.
\newblock Determinantal point processes for machine learning.
\newblock 2013.

\bibitem[KV09]{KV09}
R.~Kannan and S.~Vempala.
\newblock Spectral algorithms.
\newblock {\em Foundations and Trends in Theoretical Computer Science},
  4:157--288, 2009.

\bibitem[LJS15]{LJS15}
Chengtao Li, Stefanie Jegelka, and Suvrit Sra.
\newblock Efficient sampling for k-determinantal point processes.
\newblock 2015.

\bibitem[LPW06]{LPW06}
David~A. Levin, Yuval Peres, and Elizabeth~L. Wilmer.
\newblock {\em {{Markov Chains and Mixing Times}}}.
\newblock American Mathematical Society, 2006.

\bibitem[MT06]{MT06}
Ravi Montenegro and Prasad Tetali.
\newblock {Mathematical aspects of mixing times in {M}arkov chains}.
\newblock {\em Found. Trends Theor. Comput. Sci.}, 1(3):237--354, May 2006.

\bibitem[Nik15]{Nik15}
Aleksandar Nikolov.
\newblock Randomized rounding for the largest simplex problem.
\newblock In {\em STOC}, pages 861--870, 2015.

\bibitem[OSS11]{OSS11}
Shayan {Oveis Gharan}, Amin Saberi, and Mohit Singh.
\newblock {A Randomized Rounding Approach to the Traveling Salesman Problem}.
\newblock In {\em FOCS}, pages 550--559, 2011.

\bibitem[PP14]{PP14}
Robin Pemantle and Yuval Peres.
\newblock {Concentration of Lipschitz Functionals of Determinantal and Other
  Strong Rayleigh Measures}.
\newblock {\em Combinatorics, Probability and Computing}, 23:140--160, 1 2014.

\bibitem[RK15]{RK15}
Patrick Rebeschini and Amin Karbasi.
\newblock Fast mixing for discrete point processes.
\newblock In {\em COLT}, pages 1480--1500, 2015.

\end{thebibliography}






\end{document}